\documentclass[12pt,a4paper]{article}

\usepackage[utf8]{inputenc} 
\usepackage[T1]{fontenc}    
\usepackage{amsmath}        
\usepackage{amssymb}        
\usepackage{amsthm}         
\usepackage{mathtools}      
\usepackage{geometry}       
\usepackage{color}
\usepackage[colorlinks=true,linkcolor=blue,citecolor=blue,urlcolor=blue]{hyperref}

\theoremstyle{plain}
\newtheorem{theorem}{Theorem}[section]
\newtheorem{lemma}[theorem]{Lemma}
\newtheorem{proposition}[theorem]{Proposition}
\newtheorem{corollary}[theorem]{Corollary}
\newtheorem{prompt}[theorem]{Prompt}
\newtheorem{definition}[theorem]{Definition}
\newtheorem{question}[theorem]{Question}

\theoremstyle{definition}
\newtheorem{example}{Example}[section]
\newtheorem{remark}{Remark}[section]

\newcommand{\R}{\mathbb{R}}
\newcommand{\bC}{\mathbb{C}}
\newcommand{\E}{\mathbb{E}}

\newcommand{\op}{\mathrm{op}}

\newif\ifDRAFT
\DRAFTtrue
\ifDRAFT
\newcommand{\marrow}{\marginpar[\hfill$\longrightarrow$]{$\longleftarrow$}}
\newcommand{\niceremark}[3]
   {\textcolor{red}{\textsc{#1 #2:} \marrow\textsf{#3}}}
\newcommand{\ken}[2][says]{\niceremark{Ken}{#1}{#2}}

\newcommand{\michael}[2][says]{\niceremark{Michael}{#1}{#2}}
\newcommand{\michal}[2][says]{\niceremark{Michal}{#1}{#2}}
\newcommand{\feynman}[2][says]{\niceremark{Feynman}{#1}{#2}}
\else
\newcommand{\ken}[1]{}
\newcommand{\michael}[1]{}
\newcommand{\michal}[1]{}
\newcommand{\feynman}[1]{}
\fi

\newcommand{\ofsub}[1]{\mbox{\small \raisebox{0.0pt}{$(#1)$}}}
\newcommand{\of}[2]{{#1{\!\ofsub{#2}}}}

\ifx\BlackBox\undefined
\newcommand{\BlackBox}{\rule{1.5ex}{1.5ex}}  
\fi

\DeclareMathOperator*{\diag}{\mathop{\mathrm{diag}}}

\def\E{\mathbb E}
\def\R{\mathbb R}

\def\Pr{\mathrm{Pr}}

\let\origtop\top
\renewcommand\top{{\scriptscriptstyle{\origtop}}} 

\definecolor{silver}{cmyk}{0,0,0,0.3}
\definecolor{yellow}{cmyk}{0,0,0.9,0.0}
\definecolor{reddishyellow}{cmyk}{0,0.22,1.0,0.0}
\definecolor{black}{cmyk}{0,0,0.0,1.0}
\definecolor{darkYellow}{cmyk}{0.2,0.4,1.0,0}
\definecolor{orange}{cmyk}{0.0,0.7,0.9,0}
\definecolor{darkSilver}{cmyk}{0,0,0,0.1}
\definecolor{grey}{cmyk}{0,0,0,0.5}
\definecolor{darkgreen}{cmyk}{1,0,1,0} 

\ifx\proof\undefined
\newenvironment{proof}{\par\noindent{\bf Proof\ }}{\hfill\BlackBox\\[2mm]}
\fi

\ifx\theorem\undefined
\newtheorem{theorem}{Theorem}
\fi

\ifx\example\undefined
\newtheorem{example}{Example}
\fi

\ifx\condition\undefined
\newtheorem{condition}{Condition}
\fi
\ifx\property\undefined
\newtheorem{property}{Property}
\fi

\ifx\lemma\undefined
\newtheorem{lemma}{Lemma}
\fi

\ifx\proposition\undefined
\newtheorem{proposition}{Proposition}
\fi

\ifx\remark\undefined
\newtheorem{remark}{Remark}
\fi

\ifx\corollary\undefined
\newtheorem{corollary}{Corollary}
\fi

\ifx\definition\undefined
\newtheorem{definition}{Definition}
\fi

\ifx\conjecture\undefined
\newtheorem{conjecture}{Conjecture}
\fi

\ifx\axiom\undefined
\newtheorem{axiom}{Axiom}
\fi

\ifx\claim\undefined
\newtheorem{claim}{Claim}
\fi

\ifx\assumption\undefined
\newtheorem{assumption}{Assumption}
\fi

\ifx\condition\undefined
\newtheorem{condition}{Condition}
\fi

\title{How AI settled the complexity \\of the oldest SGD algorithm}

\author{
  Micha{\l} Derezi{\'n}ski \\
  University of Michigan \\
  \texttt{derezin@umich.edu}
  \and
  Xiaoyu Dong \\
  National University of Singapore \\
  \texttt{xdong@nus.edu.sg} 
}
\date{}

\begin{document}
\maketitle
\begin{abstract}
In 1937, Stefan Kaczmarz proposed a simple algorithm for solving systems of
linear equations. This algorithm turned out to be the
earliest known example of stochastic gradient descent, a ubiquitous computing
paradigm that drives the training of modern AI models such as ChatGPT
and Gemini. Now, those AI models have joined forces to discover the
worst-case complexity of the Kaczmarz algorithm. This paper tells the story of how it happened.
\end{abstract}

\section{Introduction}
\label{s:intro}

An essential catalyst for the remarkable breakthroughs in AI that led to the modern large language models (LLMs) such as ChatGPT and Gemini has been the algorithms used to train these models on massive datasets. While the LLM architectures have gotten progressively more complex, the training algorithms have stayed relatively simple, and in fact, they have all been based on the decades-old paradigm of stochastic gradient descent (SGD). The key idea behind SGD is that in order to minimize a certain objective function (such as an LLM's error on the training data), it suffices to access only a noisy estimate of that objective at any given time (e.g., based on a small sample of the data) while making incremental progress towards the solution. This is essential for LLM training, as the datasets have become so massive one could not hope to perform computations on everything all at once.

Commonly attributed to a 1951 paper by Robbins and Monro \cite{robbins1951stochastic}, SGD has seen a resurgence of interest over the last 20 years by AI researchers and computer scientists striving to understand its effectiveness, leading to numerous variants and extensions used in modern LLMs \cite{brown2020language,bi2024deepseek}, most notably the Adam algorithm \cite{kingma2014adam}. As a result, we have gained a robust mathematical understanding of the computational complexity of SGD algorithms in a wide range of settings (e.g., see \cite{bottou2018optimization,duchi2018introductory,asi2019importance,garrigos2023handbook}). Yet, despite this progress there is a surprising gap in the understanding of SGD: The complexity of an algorithm proposed by Stefan Kaczmarz in 1937 \cite{kaczmarz37} for solving a system of linear equations -- the oldest published example of an SGD algorithm, which predates Robbins and Monro's paper by over a decade -- has not been settled.

The Kaczmarz algorithm solves a system of $m$ linear equations with $n$ unknowns:
\begin{align*}
    a_{i1}x_1 + a_{i2}x_2 + ... + a_{in}x_n = b_i\quad \text{for}\quad i=1,2,...,m,
\end{align*}
where $x_1,x_2,...,x_n$ are the unknowns, $a_{11},a_{12},...,a_{mn}$ are the coefficients, and $b_1,b_2,...,b_m$ are the right-hand side terms.  In matrix notation, this can be written more concisely as: Given an $m\times n$ matrix $A$ and a column vector $b$ with $m$ entries, find a vector $x^*$ with $n$ entries such that $Ax^*=b$.  This task is arguably the oldest mathematical model for prediction from data, going back to the ancient times \cite{shen1999nine}, which nevertheless still finds countless applications across many areas of engineering, computer science, and applied mathematics \cite{saad2003iterative}. In order to solve such a linear system, the Kaczmarz algorithm proceeds by constructing a sequence of iterates $x^{(0)},x^{(1)},x^{(2)},...$ that gradually converges to the solution $x^*$. In the $t$-th iteration of the procedure, the algorithm picks one out of the $m$ equations and computes the solution to this equation that is the closest in Euclidean distance to the current iterate, $x^{(t-1)}$. This solution then becomes the new iterate, $x^{(t)}$, and the process continues.

The Kaczmarz algorithm implements the key idea of SGD by looking at only a single equation (one data point) in order to make progress towards the solution of the entire linear system (the dataset). The stochasticity in this process can arise only from how the algorithm picks the equations in each iteration. While Kaczmarz himself did not suggest this, picking the equations at random is crucial for ensuring reliable convergence of the method, as was shown by Strohmer and Vershynin in 2009 \cite{strohmer2009randomized}. Since their work, interest in the algorithm has exploded in the applied mathematics community, leading to many extensions and refinements, e.g., \cite{leventhal2010randomized,zouzias2013randomized,lee2013efficient,needell2014paved,gower2015randomized,liu2016accelerated,derezinski2025randomized}. In 2014, Needell, Srebro, and Ward \cite{needell2014stochastic} observed that the Kaczmarz algorithm with random selection of equations is in fact an instance of SGD, which led to a valuable transfer of ideas between the applied mathematics and computer science communities.

One of those ideas has to do with how we measure the complexity of the algorithm. Strohmer and Vershynin showed that the randomized Kaczmarz iterates converge to the solution
at a rate that scales with the condition number of $A$. Since the condition number of a matrix can be arbitrarily large even in small dimensions, this does not tell us much about the worst-case complexity of the method. 
On the other hand, Needell, Srebro, and Ward observed that 
a modified version of the algorithm which returns an average of all the iterates
can be shown to reach an arbitrarily small error $\epsilon$  after $O(1/\epsilon)$ steps, independent of the condition number of the matrix (see Section \ref{s:case-study} for details). This averaging trick turns out to be a largely superfluous modification that is commonly introduced in SGD methods purely for the sake of theoretical analysis \cite{shamir2012open}. In fact, this modification is hardly ever used in practical applications like LLM training or linear system solving, which has resulted in a substantial theory-practice gap.

Significant research effort has gone into closing this gap, with a particular focus on understanding the last-iterate convergence of SGD algorithms \cite{shamir2013stochastic,harvey2019tight,jain2019making,liu2023revisiting,zamani2025exact,ge2019step,vaswani2019fast,berthier2020tight,varre2021last,wu2022last,liu2023aiming}. Surprisingly, the Kaczmarz algorithm has resisted these efforts, being one of the last remaining simple problem settings that is unresolved in this context, and recent attempts by both applied mathematicians \cite{steinerberger2023approximate} and computer scientists \cite{evron2025continual,attia2025fast,derezinski2026last} have only led to partial results.

We resolve this problem by showing that the last iterate of the Kaczmarz algorithm with randomly selected equations attains the optimal $O(1/\epsilon)$ convergence rate.

\begin{theorem}
    Given matrix $A$ and vector $b$ such that $Ax^*=b$ for some $x^*$, the Kaczmarz algorithm with randomly selected equations 
    reaches $\epsilon$ error \eqref{eq:average-iterate}
    after $O(1/\epsilon)$ iterations.
\end{theorem}

Perhaps more remarkable than the result itself is how we obtained it. In what could be described as a full circle, the complexity of the oldest SGD algorithm was settled through a collaboration between Gemini Deep Think and ChatGPT Pro, which was only initiated and facilitated by the authors. In this process, Gemini discovered a novel connection to a branch of functional analysis that was previously unfamiliar to the authors, which allowed ChatGPT to derive an elementary proof of the claim. The resulting theorem should be of broader interest to the machine learning literature, and in particular, it already resolves long-standing open questions regarding SGD convergence 
that are not specific to the Kaczmarz algorithm \cite{berthier2020tight,varre2021last}. 

In Section \ref{s:case-study} we give an account of how the result was obtained, as we hope that this will be a useful case study for what the future of AI-led research might look like. In Section \ref{s:proof}, we give the full proof of the result, which was independently verified by the authors. Our concluding thoughts are given in Section \ref{s:conclusions}. We remark that while the result was proven by AI, the paper itself was written by the authors without the use of~AI.

\section{AI-led research: A case study}
\label{s:case-study}

In this section, we provide a high-level account of our interactions with AI that led to the main result of this paper. Along the way, we introduce sufficient mathematical notation to ground the problem and the surrounding discussion. Basic knowledge of linear algebra and probability at an undergraduate level should be sufficient to follow this section.

\paragraph{What is AI-led research?}
Our case study is one of several recent works showcasing the use of AI in mathematical research. Those works include efforts at understanding the capabilities of AI as a fully autonomous problem solving agent \cite{abouzaid2026first,alon2026remarks}, as well as accounts of AI-assisted research \cite{jang2025point,woodruff2026accelerating,alexeev2026forbidden}, where these tools are used to help a human researcher with some parts of the analysis. Our project occupies a middle ground which we call ``AI-led research'', where AI serves as a primary driver of the analysis but is not fully autonomous (in fact, attempts at a fully autonomous AI result failed in our case, as discussed below), thereby keeping the human researcher in the loop. 

\paragraph{Background and problem formulation.} We start our discussion by formally defining the Kaczmarz algorithm. For simplicity, following the original paper \cite{kaczmarz37} we assume that the algorithm is solving a square linear system, i.e., the number of equations $m$ is the same as the number of unknowns $n$. Also, when randomly selecting the equations, we use the original importance sampling scheme proposed by Strohmer and Vershynin \cite{strohmer2009randomized}. Neither of these choices has a significant impact on the analysis, as discussed in Section~\ref{s:proof}. 

\begin{definition}\label{d:kaczmarz}
Consider the randomized Kaczmarz algorithm, which given an $n\times n$ invertible matrix $A$ with rows $a_1^\top,...,a_n^\top$, and an $n$-dimensional vector $b$, starting from an all-zeros $n$-dimensional vector $x^{(0)}$ performs $T$ iterations of the following procedure:
$$x^{(t+1)} = x^{(t)} - \frac{a_{i_t}^\top x^{(t)}- b_{i_t}}{\|a_{i_t}\|^2}a_{i_t}$$
for a randomly sampled $i_t\in\{1,...,n\}$ drawn with probability proportional to $\|a_{i_t}\|^2$ independently in each iteration. The algorithm returns the last iterate $x^{(T)}$.
\end{definition}

Here, $\|a_i\| = \sqrt{\sum_ja_{ij}^2}$ denotes the Euclidean norm of a vector $a_i$. In order to measure how well the Kaczmarz estimate $x^{(T)}$ approximates the exact solution to the linear system, it is natural to consider how well it models its right-hand side via the so-called squared error, $\|Ax^{(T)}-Ax^*\|^2$, where $x^*=A^{-1}b$ is the exact solution. This metric is also particularly significant in the broader context of SGD analysis, as it defines the objective function 
in which Kaczmarz is an instance of SGD. Concretely suppose that our goal is to minimize the quadratic $f(x) = \frac12\|A(x - x^*)\|^2$. Given an estimate $x^{(t)}$, the \emph{gradient descent} method updates this iterate in the direction of steepest descent of the objective at $x^{(t)}$, i.e., the direction opposite to the gradient $\nabla f(x^{(t)})$, returning $x^{(t)}- \eta \nabla f(x^{(t)})$ where $\eta$ is some chosen step size. A \emph{stochastic} gradient descent method is one that uses a random direction $g^{(t)}$ in place of the gradient, while matching it in expectation, i.e., such that $\E[g^{(t)}] = \eta\nabla f(x^{(t)})$. In the case of randomized Kaczmarz, $g^{(t)} = \frac{a_{i_t}^\top x^{(t)}- b_{i_t}}{\|a_{i_t}\|^2}a_{i_t}$, and a simple calculation shows that this is indeed a stochastic gradient step:
\begin{align*}
  \E[g^{(t)}] = \sum_i \Pr[i_t = i]\frac{a_{i}^\top x^{(t)}- b_{i}}{\|a_{i}\|^2}a_{i}\ \propto\ \sum_i (a_{i}^\top x^{(t)}- b_{i})a_i = \nabla f(x^{(t)}).
\end{align*}

The observation that motivates our question is: If, instead of $\hat x = x^{(T)}$, the algorithm were to return the average of all of the iterates, $\hat x= \frac1T\sum_{t=1}^Tx^{(t)}$, then the classical SGD analysis (e.g., \cite{bach2013non,needell2014stochastic}) shows that the expected squared error of this estimate satisfies:
\begin{align}
    \E\,\|A(\hat x - x^*)\|^2 \leq \epsilon \cdot \|A\|_F^2\|x^*\|^2\quad\text{after $T=O(1/\epsilon)$ iterations.}\label{eq:average-iterate}
\end{align}
Here, $\|A\|_F = \sqrt{\sum_j\|a_j\|^2}$ denotes the Frobenius norm, which is a natural extension of the Euclidean norm. The factors on the right-hand side ensure homogeneity of the bound. The question we are interested in, following, e.g., \cite{steinerberger2023approximate}, is: 
\begin{question}\label{q:main}
Does \eqref{eq:average-iterate} hold for $\hat x = x^{(T)}$, the last iterate of randomized Kaczmarz? 
\end{question}
Beyond seeking to understand the classical Kaczmarz algorithm, this question has been used as perhaps the simplest model for studying the dynamics of SGD in LLM training \cite{varre2021last} and related learning models \cite{evron2025continual}, as well as as a key example of the theory-practice gap in stochastic optimization \cite{attia2025fast}. Indeed, as discussed in Section \ref{s:proof}, by answering this question we were able to resolve a broader open problem in last-iterate SGD analysis for noiseless least squares regression \cite{berthier2020tight}.


\paragraph{Initial result via AI-assisted research.} Before we describe how the main result of this paper was obtained, we briefly discuss our initial efforts towards resolving Question~\ref{q:main}, which led to a partial result published in an earlier paper \cite{derezinski2026last}. In that project, we did not set out to settle the question with AI. Instead, using old-fashioned human research we reduced the problem to showing an elementary discrete math inequality.
\begin{proposition}[\cite{derezinski2026last}, human proof]\label{p:initial-human}
    Let $\alpha,K>0$ for all $t\geq K$, $\rho\in(0,1/2]$ satisfy
  \begin{align}
    \rho(1-2\rho)^t +
    K\rho\sum_{i=1}^t\frac{(1-2\rho)^{t-i}}{i^\alpha} \leq \frac{K}{(t+2)^\alpha}.\label{eq:summation-bound}
  \end{align}
  Then, there exists an absolute constant $C>0$ such that randomized Kaczmarz satisfies:
  \begin{align*}
   \mathbb{E}\,\|A(x^{(T)} - x^*)\|^2 \leq \frac C{T^\alpha}\|A\|_F^2\|x^*\|^2.   
  \end{align*}
\end{proposition}
In particular, Proposition \ref{p:initial-human} implies the desired error bound after $O(1/\epsilon^{1/\alpha})$ iterations, so $\alpha=1$ would resolve Question \ref{q:main}. While finding the largest $\alpha$ that satisfies \eqref{eq:summation-bound} turned out to be surprisingly challenging, it is clearly a perfect candidate for an AI-assisted proof thanks to its elementary nature (e.g., no need to describe an algorithm). We posed this problem to ChatGPT Pro (versions 5.1 and 5.2, latest at the time), and after multiple rounds of discussion (which involved mainly pointing out errors and suggesting change of directions when the methods seemed hopeless), we obtained the following claim. We refer to \cite{derezinski2026last} for the details.
\begin{proposition}[\cite{derezinski2026last}, AI proof]\label{p:initial-ai}
Setting $\alpha=3/4$ and $K=300$ satisfies inequality \eqref{eq:summation-bound}.
\end{proposition}
Naturally, this does not resolve Question \ref{q:main}, but Propositions \ref{p:initial-human} and \ref{p:initial-ai} gave the strongest partial result at the time. Did the sub-optimality lie with the human reduction or the AI proof? It turns out the AI argument was nearly sharp: No $\alpha>3/4+0.003=0.753$ can satisfy \eqref{eq:summation-bound}. It was the human argument that led to the sub-optimal result, so the next logical step was to try and take the human out of the~loop. 

\paragraph{Attempts at a fully autonomous AI solution.} For the remainder of the project, we relied on the following two publicly available versions of AI reasoning models: Gemini Deep Think 3.1 and ChatGPT Pro 5.5. We note that the outputs of both models vary substantially from run to run, and also the models appear to be frequently updated behind the scenes, so little reproducibility is possible.

In order to see whether either of the models can directly resolve Question \ref{q:main} without human involvement, we constructed the following prompt, partly inspired by \cite{unit-distances}.

\begin{prompt}\label{p:main-prompt}
Resolve the following problem completely: Determine whether there exists a constant $C>0$ such that for any invertible $n\times n$ matrix $A$ and $n$-dimensional vector $b$, the vector $x^{(T)}$ returned by the Randomized Kaczmarz algorithm satisfies:
$$ \mathbb{E}\,\|A x^{(T)} - b\|^2 \leq \frac CT\|A\|_F^2\|A^{-1}b\|^2. $$
A complete solution must prove exactly one of the following:
\begin{itemize}
\item Affirmative solution: Give a proof of the existence of such constant $C>0$ that is independent of $A$, $b$, and $T$.
\item Negative solution: Prove that no such constant $C>0$ exists by constructing a corresponding family of counter-examples.
Partial progress does not count unless it implies one of the two resolutions above. In particular, heuristic evidence is insufficient.
\end{itemize}
\end{prompt}
The full prompt we used included Definition  \ref{d:kaczmarz}, Prompt \ref{p:main-prompt}, and a short ``guardrails'' paragraph encouraging sufficient mathematical rigor of the output. We later discovered a small loophole in our prompt: The problem leaves it slightly ambiguous whether the constant $C$ is allowed to depend on the matrix dimension $n$. Naturally, we intended for $C$ to be an absolute constant, and we expect that most human experts would interpret the prompt as we did, but in some of our interactions Gemini chose to interpret $C$ as being dependent on $n$ (which simplifies the problem considerably). Nevertheless, this was not an issue the first time we prompted either model, which is what we report below.
\bigskip

\textbf{ChatGPT's answer.} After $93$ minutes of thinking, ChatGPT responded with:
\begin{quote}
    \it I cannot honestly provide the requested ``complete resolution'' proving either the affirmative $O(1/T)$ last-iterate bound or a counterexample family. The problem, in the stated last-iterate form, appears to be a current open worst-case question rather than a
settled theorem.
\end{quote}
It then proceeded to cite our previous paper \cite{derezinski2026last}, including the sub-optimal rate and the average iterate guarantee, discussing why those results do not resolve the question one way or the other.
\bigskip

\textbf{Gemini's answer.} After several hours, Gemini responded as follows:
\begin{quote}
\it%
\!The solution is Affirmative. \!There exists a universal absolute constant $C>0$ (strictly independent of the matrix $A$, target vector $b$, dimension $n$, and~iteration count $T$) such that the expected squared residual satisfies $\mathbb{E}\|A x^{(T)} - b\|^2 \leq \frac{C}{T}\|A\|_F^2\|A^{-1}b\|^2$.
\end{quote}
It then proceeded with its argument, which we discuss in more detail below. Unfortunately, that argument cannot be described as a complete proof, since it abruptly starts making unsupported claims and including some incomplete references to the literature. Further prompts for clarification led to less coherent outputs and more incomplete references, which led us to conclude that Gemini did not provide a valid proof. Overall, we conclude that a fully autonomous AI solution to Question \ref{q:main} was not successful.

\paragraph{Gemini's insights and blunders.} While Gemini's argument was far from complete, it included some insights which proved crucial. To explain them, let us state the following standard recursion for the Kaczmarz error which follows directly from the definition:
\begin{align*}
x^{(t+1)} - x^* = \Big(I_n - \frac{a_{i_t}a_{i_t}^\top}{\|a_{i_t}\|^2}\Big)\big(x^{(t)}-x^*\big),
\end{align*}
where $I_n$ is the $n\times n$ identity matrix.
In particular, this formula can be used to define a sequence of matrices $M_t$  that describes the evolution of the expected error:
\begin{align}
  \E\|A(x^{(T)}-x^*)\|^2 &= (x_0-x^*)^\top M_{T}(x_0-x^*),\label{eq:M_T}\\
  \text{where}\quad M_0=A^\top A,\quad M_{t+1} &= \E\bigg[\Big(I_n - \frac{a_{i_t}a_{i_t}^\top}{\|a_{i_t}\|^2}\Big) M_t \Big(I_n - \frac{a_{i_t}a_{i_t}^\top}{\|a_{i_t}\|^2}\Big) \bigg].\label{eq:Mrec}
\end{align}
Since $x_0=0$, using elementary linear algebra it follows from \eqref{eq:M_T} that to resolve Question~\ref{q:main} it suffices to bound the operator norm of $M_T$ by $C\|A\|_F^2/T$. Recall that for a real symmetric matrix $M$, its operator norm is given by:
\begin{align*}
  \|M\|_{\op} := \sup_{v\ne 0} \frac{|v^\top M v|}{\|v\|^2}.
\end{align*}

Since the distribution of the indices $i_t$ is the same in every iteration, recursion \eqref{eq:Mrec} can be unfolded concisely, obtaining $M_T = \Phi(M_{T-1}) = \Phi(\Phi(M_{T-2}))= ... = \Phi^T(M_0)$, where $\Phi(X):=\E_i\big(I_n-\frac{a_{i}a_i^\top}{\|a_i\|^2}\big)X\big(I_n-\frac{a_{i}a_i^\top}{\|a_i\|^2}\big)$  is a linear operator over the space of symmetric matrices, with the expectation taken over index $i$ distributed identically to $i_t$ in Definition~\ref{d:kaczmarz}. We use $\Phi^T$ to denote the repeated composition of $\Phi$ with itself. Also, a direct calculation shows that $M_0 = A^\top A = \|A\|_F^2(I_n - \Phi(I_n))$. Altogether, 
\begin{align*}
  M_T = \Phi^T(A^\top A) = \|A\|_F^2\Phi^T(I_n - \Phi(I_n)) = \|A\|_F^2(\Phi^T - \Phi^{T+1})(I_n).
\end{align*}
All of this was derived independently by Gemini, although we did use something very similar to the operator $\Phi$ in our previous work. At this point, Gemini evidently made a connection between the linear operator $\Psi:=\Phi^T - \Phi^{T+1}$ and similar constructs that appear in a branch of functional analysis literature. In particular, it observed that to bound $M_T$ in the operator norm, it suffices to bound $\Psi$ in an appropriately induced norm:
\begin{align*}
  \|\Psi\|_{\op\to\op} := \sup_{X\ne 0} \frac{\|\Psi(X)\|_\op}{\|X\|_\op}.
\end{align*}
Following this observation, Gemini produced the following paragraph:
\begin{quote}
  \it
  The linear map $\Phi$ belongs strictly to the class of \textit{symmetric quantum Markov semigroups}. Specifically, $\Phi$ is completely positive (represented in Kraus form by symmetric operators $K_i = \sqrt{p_i}P_i$), sub-unital ($\Phi(I) \preceq I$), and trace-symmetric. 
It is a comprehensively established theorem in operator algebra that symmetric Markov operators defined on non-commutative $C^*$-algebras intrinsically exhibit discrete analyticity; they act as Ritt operators. By the non-commutative Stein interpolation theorems and spectral multiplier estimates for symmetric Markov semigroups (see, for instance, Junge, Le Merdy, and Xu, 2006, ``$H^\infty$ functional calculus and square functions on noncommutative $L_p$-spaces''), the discrete time-derivative of the semigroup satisfies a uniform, bounded sectorial decay in the completely bounded $L^\infty$-norm:
$$ \left\| \Phi^T - \Phi^{T+1} \right\|_{\op \to \op} \leq \frac{C}{T}. $$
\end{quote}
If the reader feels confused by the terminology and mathematical jargon, then rest assured, so were the authors. While the claimed inequality at the end of the paragraph would indeed resolve our question, the justification is woefully insufficient. When asked which specific result of Junge, Le Merdy, and Xu \cite{junge2006h} is being cited, Gemini pointed to Theorem 4.1, and gave a statement of that theorem. In fact, there is no Theorem 4.1 in \cite{junge2006h}, and the statement did not match any of the results present in that work. When pushed further, Gemini produced another reference from functional analysis, a paper by Blunck \cite{blunck2001analyticity}. This time, it was clear that the reference does indeed study bounds of the form $\|\Phi^T-\Phi^{T+1}\|\leq C/T$ for various classes of operators in various norms, however the specific norm $\|\cdot\|_{\op\to\op}$ is not directly bounded there.

\paragraph{ChatGPT completes the argument.} While Gemini's answer was problematic, it did clearly uncover a promising avenue for attacking the problem. We synthesized Gemini's observations and the conjectured operator bound into a short document.
Then we gave this document to ChatGPT and prompted it with basic questions including checking whether the operator bound holds using a numerical simulation and whether its proof follows directly from the references given by Gemini.

In response, ChatGPT first confirmed that the operator bound is supported by numerical evidence, and also stated that the operator bound, while closely related to the provided references, does not directly follow from any of the results that they contain. However, when prompted further, ChatGPT simply gave its own proof of the bound, 
thereby resolving Question \ref{q:main}. We note that after our initial round of communication
with AI (which led to the complete solution), ChatGPT was later sometimes able to
provide valid proofs of the original claim from certain prompts that
did not include Gemini's hints. This may be due to knowledge retention from previous chats,  randomness, or internal updates.

Curiously, while ChatGPT's proof is remarkably elementary, it does rely on some techniques which are common in functional analysis and operator theory, but are not typically seen in the convergence analysis of SGD algorithms. This suggests that Gemini's observations worked as a crucial hint to complete the argument. The detailed proof, which was rewritten for clarity but remains largely unchanged, is given in the following section.

\section{Main result and the AI proof}
\label{s:proof}
In this section, we state the main technical result shown by ChatGPT based on the idea of Gemini. We then discuss its implications for Kaczmarz and SGD. We conclude the section with the proof of the main result.

\paragraph{Notation.} Let $\mathcal S_n$ denote the space of $n\times n$ symmetric real matrices equipped with the operator norm. The $n \times n$ identity matrix is denoted by $I_n$, and we also use the simplified notation $I$ when there is no risk of confusion about the dimension. We say that $M\in \mathcal S_n$ is positive semidefinite, denoted by $M\succeq 0$, if $x^\top Mx\ge 0$ for all $x\in\R^n$, and we write $M\preceq N$ if $N-M\succeq 0$ (Loewner ordering of symmetric matrices). A matrix $B\in\mathcal S_n$ is called a positive contraction if $0\preceq B\preceq I_n$. 

\begin{theorem}\label{thm:main}
Let $B$ be a measurable random positive contraction in $\mathcal S_n$, and let
\[
        \Phi(X)=\E[BXB],\qquad X\in\mathcal S_n.
\]
Then, for every integer $T\ge0$, we have
\[
        \|\Phi^T-\Phi^{T+1}\|_{\op\to\op}
        \le \frac{3}{T+1},
\]
where $\Phi^T$ is the composition of $T$ copies of $\Phi$ (i.e.,  $\Phi^{t+1}(X)=\Phi(\Phi^{t}(X))$ for $t=1,2,...$).

\end{theorem}

\subsection{Implications for Kaczmarz and SGD}
As implied by equation \eqref{eq:M_T} of Section \ref{s:case-study}, the expected error of the Kaczmarz algorithm can be bounded as follows:
\begin{align*}
  \E\,\|A(x^{(T)}-x^*)\|^2 \le \|\Phi^T(A^\top A)\|_{\op}\cdot\|x^*\|^2,
\end{align*}
where $\Phi(X)=\E[BXB]$ for $B = I_n - \frac{a_ia_i^\top}{\|a_i\|^2}$ with $i$ distributed according to how we select the random equations in the algorithm. It is easy to verify that $B$ is an orthogonal projection matrix, which in particular makes it a positive contraction, and thus $\Phi$ falls under the conditions of Theorem \ref{thm:main}. When the index $i$ follows the distribution from Definition \ref{d:kaczmarz}, i.e., $\Pr[i] = \|a_i\|^2/\|A\|_F^2$, then we also have:
\begin{align}
  \Phi(I_n) = \E[B^2] = \E[B] = I_n - \sum_i \frac{\Pr[i]}{\|a_i\|^2}  a_ia_i^\top= I_n - \|A\|_F^{-2}A^\top A,\label{eq:Phi_I}
\end{align}
which in turn implies that $\Phi^T(A^\top A) = \|A\|_F^2\Phi^T(I_n - \Phi(I_n)) = \|A\|_F^2(\Phi^T-\Phi^{T+1})(I_n)$. Since $\|(\Phi^T-\Phi^{T+1})(I_n)\|_\op\leq \|\Phi^T-\Phi^{T+1}\|_{\op\to\op}$, Theorem \ref{thm:main} immediately yields the following corollary:
\begin{corollary}
  The randomized Kaczmarz algorithm after $T\ge 3/\epsilon$ iterations satisfies:
  \begin{align*}
    \E\,\|A(x^{(T)}-x^*)\|^2\leq \epsilon\cdot\|A\|_F^2\|x^*\|^2.
    \end{align*}
  \end{corollary}
  This resolves Question \ref{q:main} and matches the hypothesis raised by Steinerberger \cite{steinerberger2023approximate} regarding the convergence of Strohmer and Vershynin's randomized Kaczmarz \cite{strohmer2009randomized}. Other recent works considered a version of the question were the importance sampling is replaced by uniform sampling, i.e., where $Pr[i] = 1/n$ for all $i$. Here, the only difference rests on the calculation in \eqref{eq:Phi_I}, and we can still show a weaker bound $\Phi(I_n) \preceq I_n - (n\|A\|_\op^2)^{-1}A^\top A$. Plugging this into the upper bound, we get $\E\,\|A(x^{(T)}-x^*)\|^2\leq \frac{3n}T\|A\|_\op^2\|x^*\|^2$, which resolves the question of Attia, Schliserman, Sherman, and Koren~\cite{attia2025fast}.

Theorem \ref{thm:main} applies much more broadly than for the Kaczmarz algorithm. To illustrate this, we show how it resolves an open question in the last iterate analysis of SGD for the noiseless least squares regression task, which is a mathematical model for studying SGD dynamics on LLMs. Following \cite{berthier2020tight,varre2021last}, we consider the task of minimizing:
\begin{align*}
          \mathcal{R}(\theta) := \frac12\E_{(X,Y)}\,\big(\langle\theta, X\rangle - Y\big)^2\quad\text{over}\quad\theta\in\mathcal \R^d,
\end{align*}
where $\langle\theta, X\rangle$ denotes the inner product, $(X,Y)\in \mathcal \R^d\times \R$ are jointly distributed random variables such that there exists $\theta^*\in\R^d$ satisfying $\langle\theta^*,X\rangle=Y$ almost surely, and for some $R\ge 0$, the following fourth moment condition is satisfied:
\begin{align*}
          \E\big[\|X\|^2 XX^\top\big] \preceq R\cdot \E[XX^\top].
\end{align*}
This problem is sometimes further generalized from $\R^d$ to infinite dimensional Hilbert spaces. We will focus on the finite-dimensional case for simplicity of presentation, but the extension to Hilbert spaces should be straightforward. Given a sequence of observations $(X_t,Y_t)$ drawn i.i.d.\ from $(X,Y)$, the SGD algorithm for minimizing $\mathcal{R}$ proceeds as follows:
\begin{align*}
    \theta^{(t+1)} = \theta^{(t)} - \gamma \big(\langle\theta^{(t)}, X_t\rangle-Y_t\big)X_t,
\end{align*}
where $\gamma$ is the step size. Obtaining an $O(1/T)$ convergence rate for the last iterate $\theta^{(T)}$ in this model has been a long-standing question in SGD analysis, which was first posed by Berthier, Bach, and Gaillard \cite[Remark 1]{berthier2020tight}, with partial progress obtained by \cite{varre2021last}. Theorem \ref{thm:main} resolves this question, matching the lower bound of \cite{berthier2020tight} up to constants.
\begin{corollary}
In the noiseless regression model, SGD with step size $\gamma=1/R$ satisfies
\begin{align*}
          \E\, \mathcal R(\theta^{(T)}) \le \frac{3R\|\theta^*\|^2}{T+1}.
\end{align*}
\end{corollary}
  
\subsection{Proof of the main result}

To explain the idea of the proof of Theorem~\ref{thm:main}, we first consider the scalar case. Suppose $\Phi$ is not an operator but a scalar $0 \le y \le 1$, and we want to bound $|y^{T}-y^{T+1}|=(1-y)y^{T}$. Then we can simply use monotonicity $(1-y)y^{T} \le (1-y)y^{T-1} \le ... \le (1-y)$ to bound $|y^{T}-y^{T+1}|$ by the average of $(1-y)y^{t}$, i.e., 
\begin{align*}
    |y^{T}-y^{T+1}|=(1-y)y^{T} \le \frac{1}{T+1}\left((1-y)y^{T}+(1-y)y^{T-1}+...+(1-y) \right).
\end{align*}
Then, observe that the right hand side can be rewritten as a telescoping sum, i.e., 
\begin{align*}
    &\frac{1}{T+1}\left((1-y)y^{T}+(1-y)y^{T-1}+ \cdots +(1-y) \right)\\=&\frac{1}{T+1}\left((1-y)+(1-y)y+ \cdots +(1-y)y^{T} \right)
    \\=&\frac{1}{T+1}\left((1-y)+(y-y^2)+ \cdots +(y^{T} -y^{T+1})\right)
        =\frac{1}{T+1}(1-y^{T+1}) \le \frac{1}{T+1}.
\end{align*}

When $\Phi$ is an operator, most of the above calculations will not make sense, but the idea of bounding $\|\Phi^T-\Phi^{T+1}\|_{\op\to\op}$ by a notion of average can still lead to a solution. To explain how to write $\Phi^T-\Phi^{T+1}$ as an average, we first prove the following more explicit representation of $\Phi^T$.

\begin{lemma}[Product representation]\label{lem:iterates} Let $B$ and $\Phi$ be as in Theorem~\ref{thm:main}. Let $B_1,B_2,\ldots$ be independent copies of $B$.  For $T\ge1$, let
\begin{align*}
    W_0 = I_n\quad\text{and}\quad W_T=B_TB_{T-1}\cdots B_1.
\end{align*}
Then, for every $T\ge0$ and every $X\in\mathcal S_n$, we have
\[
        \Phi^T(X)=\E[W_TXW_T^\top].
\]
with the convention that $\Phi^0$ is the identity operator, i.e., $\Phi^0(X)=X$ for all $X\in\mathcal S_n$.
\end{lemma}

\begin{remark}
The matrix $W_T$ may not be symmetric, since the factors need not commute.
\end{remark}

\begin{proof} 
We use induction. The result holds directly when $T=0,1$. Now, assume the formula holds for $T=t$, and we consider the case when $T=t+1$.  Since $B_{t+1}$ is independent of $B_1,\ldots,B_t$, by iterated expectation, i.e., Fubini's theorem, we have
\begin{align*}
        \Phi^{t+1}(X)&=\Phi(\Phi^t(X)) \\
        &=\E_{B_{t+1}}\left[B_{t+1}\E_{B_1,\ldots,B_{t}}[W_tXW_t^\top]B_{t+1}\right] \\
        &=\E_{B_{t+1}}\left[\E_{B_1,\ldots,B_{t}}[B_{t+1}W_tXW_t^\top B_{t+1} ]\right] \\
&=\E_{B_1,\ldots,B_{t+1}}\left[B_{t+1}W_tXW_t^\top B_{t+1}\right] \\
        &=\E[W_{t+1}XW_{t+1}^\top],
\end{align*}
which concludes the proof.

\end{proof}

The first key idea in the proof of Theorem~\ref{thm:main} is to rewrite $\Phi^T-\Phi^{T+1}$ as a leave-one-out average, which can then be upper-bounded deterministically.

\begin{lemma}[Leave-one-out identity]\label{lem:loo} Fix $t\ge1$. Let $B$, $\Phi$,   $B_1,B_2,\ldots$, and $W_T$ be as in Lemma~\ref{lem:iterates}. For $1\le t\le T+1$, define the product with the $t$-th factor omitted by
\[
        O_t=B_{T+1}\cdots B_{t+1}B_{t-1}\cdots B_1,
\]
with the usual convention that an empty product is $I$.
Then, for any $X\in\mathcal S_n$,
\[
        (\Phi^{T}-\Phi^{T+1})(X)
        =\frac1{T+1}\,\E\sum_{t=1}^{T+1}
        \left(O_tXO_t^\top-W_{T+1}XW_{T+1}^\top\right).
\]
\end{lemma}

\begin{remark}
    In the scalar case, the leave-one-out identity is simplified to $y^{T}-y^{T+1}=\frac{1}{T+1} \big(\sum_{t=1}^{T+1} y_{T+1} \cdots (1-y_t) \cdots y_1\big)$ where $y_1=y_2= \cdots =y_{T+1}$, which is straightforward.
\end{remark}

\begin{proof} 
By Lemma~\ref{lem:iterates}, we have $\Phi^{T+1}(X)=\E[W_{T+1}XW_{T+1}^\top]$.  For a fixed $t$ with $1\le t\le T+1$, the finite sequence of random matrices
\[
        B_{T+1},\ldots,B_{t+1},B_{t-1},\ldots,B_1
\]
are $T$ i.i.d. copies of $B$.  As a result, $O_t$ has the same law of distribution as $W_{T}$, since both of them are an (ordered) product of $T$ independent copies of $B$.  Therefore, we conclude
\[
        \E[O_tXO_t^\top]=\Phi^{T}(X).
\]
The desired result follows by averaging the above identity over $t=1,\ldots,T+1$ and subtracting $\Phi^{T+1}(X)$.
\end{proof}

Now it remains to bound the leave-one-out sum $\sum_{t=1}^{T+1}
\left(O_tXO_t^\top-W_{T+1}XW_{T+1}^\top\right)$, which is the main technical difficulty in the proof of Theorem~\ref{thm:main}. Remarkably, this bound is fully deterministic since we view $B_1,\ldots,B_{T+1}\in\mathcal S_n$ as arbitrary matrices. This is possible because the randomness in the problem was utilized in the averaging step of the leave-one-out identity.

\begin{lemma}[Deterministic bound]\label{lem:pathwise} 
Fix $T \ge 0$ and a sequence of arbitrary positive contractions $B_1,\ldots,B_{T+1}\in\mathcal S_n$. Let
\[
        W=B_{T+1}B_{T}\cdots B_1\quad\text{and}\quad
        O_t=B_{T+1}\cdots B_{t+1}B_{t-1}\cdots B_1,
\]
for $1\le t\le T+1$.
Then, for every $X\in\mathcal S_n$, we have
\[
        \Big\|
        \sum_{t=1}^{T+1}\left( O_tX O_t^\top-WXW^\top\right)
        \Big\|_\op 
        \le 3\|X\|_\op.
\]
\end{lemma}
It is worth noting that ChatGPT's proof of Lemma \ref{lem:pathwise}, given below in Section \ref{s:pathwise}, relies on tools from complex analysis such as the Cauchy estimate for the derivative of an analytic function. The use of these techniques (which to our knowledge is not typical for SGD analysis) suggests deeper connections with the functional analysis literature \cite{blunck2001analyticity,junge2006h} that was provided to ChatGPT from Gemini's insights. 

Theorem \ref{thm:main} now follows immediately by combining Lemmas \ref{lem:loo} and \ref{lem:pathwise}.
\begin{proof}[Proof of Theorem~\ref{thm:main}]

Fix $T\ge0$ and $X\in\mathcal S_n$.
By Lemma~\ref{lem:loo} and 
 Jensen's inequality, we have
\begin{align*}
        \|(\Phi^T-\Phi^{T+1})(X)\|_\op
        =&\Big\|\frac1{T+1}\,\E\sum_{t=1}^{T+1}
        \left( O_tX O_t^\top-W_{T+1}XW_{T+1}^\top\right)\Big\|_\op \\\le& \frac1{T+1}\,\E \Big\|\sum_{t=1}^{T+1}
        \left( O_tX O_t^\top-W_{T+1}XW_{T+1}^\top\right)\Big\|_\op.
\end{align*}
For each realization of $B_1,\ldots,B_{T+1}$, Lemma~\ref{lem:pathwise} shows that
\[
        \Big\|
        \sum_{j=1}^{T+1}
        \left( O_tX O_t^\top-W_{T+1}XW_{T+1}^\top\right)
        \Big\|_\op
        \le3\|X\|_\op.
\]
Putting these together, we conclude that \begin{align*}
    \|(\Phi^T-\Phi^{T+1})(X)\|_\op\leq \frac{3}{T+1}\|X\|_\op.\end{align*} 
Since the above bound is uniform over $X\in\mathcal S_n$, taking a supremum over $X\in\mathcal S_n$ with $X \ne 0$, we have
    \begin{align*}
  \|(\Phi^T-\Phi^{T+1})\|_{\op\to\op} = \sup_{X\in\mathcal S_n, X\ne 0} \frac{\|(\Phi^T-\Phi^{T+1})(X)\|_\op}{\|X\|_\op} \le \frac{3}{T+1}.
\end{align*}
\end{proof}

\subsection{Deterministic leave-one-out bound via complex analysis}
\label{s:pathwise}
It remains to prove Lemma \ref{lem:pathwise}. Recall that our goal is to bound $\sum_{t=1}^{T+1}\left( O_tX O_t^\top-WXW^\top\right)$, where
$W=B_{T+1}B_{T}\cdots B_1$ and $ O_t=B_{T+1}\cdots B_{t+1}B_{t-1}\cdots B_1$.
Defining the first-order difference term $A_t = O_t - W$, we can rewrite each summand as follows: 
\begin{align*}
        O_tX O_t^\top-WXW^\top
        &= (W+A_t)X(W+A_t)^\top-WXW^\top \\
        &= WXA_t^\top+A_tXW^\top+A_tXA_t^\top.
\end{align*}
We can now break the sum down into three components, which we will handle separately:
\begin{align}
        \sum_{t=1}^{T+1}( O_tX O_t^\top-WXW^\top)
        &=WX\Big(\sum_{t=1}^{T+1} A_t\Big)^\top
          +\Big(\sum_{t=1}^{T+1} A_t\Big)XW^\top 
        +\sum_{t=1}^{T+1} A_tXA_t^\top.\label{eq:decomposition}
\end{align}
Since $\|W\|_\op\le1$, it suffices to bound $\|\sum_j A_t\|_\op$ and $\|\sum_{t=1}^{T+1} A_tXA_t^\top\|_\op$, which we refer to as the first-order and quadratic terms, respectively. We start with the former.

\begin{lemma}[First-order bound]\label{lem:first-coefficient}
Let $B_1,\ldots,B_{T+1}\in\mathcal S_n$ be positive contractions, and let $A_t= O_t-W$, where $ O_t$ and $W$ are defined as in Lemma~\ref{lem:pathwise}. Then,
\[
         \Big\|\sum_{t=1}^{T+1} A_t\Big\|_\op = \Big\|\sum_{t=1}^{T+1}
          B_{T+1}\cdots B_{t+1}(I_n-B_t)B_{t-1}\cdots B_1\Big\|_\op\leq 1.
\]
\end{lemma}
\begin{proof}[Proof sketch]
The proof relies on tools from complex analysis, which we explain  here in a simplified scalar setting, where given  $0 \le y_t \le 1$ we want to bound
\begin{align*}
    \Big|\sum \limits_{t=1}^{T+1} y_{T+1} \cdots  y_{t+1} (1-y_t) y_{t-1} \cdots y_1\Big| \le 1.
\end{align*}
Define function $f(z)=(y_{T+1}+z(1-y_{T+1}))(y_{T}+z(1-y_{T})) \cdots (y_{1}+z(1-y_{1}))$ over complex variables $z\in\bC$. By product rule of complex differentiation,
\begin{align*}
    f'(0)=\sum \limits_{t=1}^{T+1} y_{T+1} \cdots y_{t+1} (1-y_t)  y_{t-1} \cdots y_1.
\end{align*}
Thus, it suffices to bound $|f'(0)|$. Observe that $f(z)$ is analytic on the entire complex plane because it is a product of linear complex functions, and also $|f(z)| \le 1$ on the circle $|z|=1$ because each factor satisfies $|y_{t}+z(1-y_{t})| \le 1$ as $0 \le y_t \le 1$. We can use this to bound the derivative at the center of this circle by using Cauchy's estimate (Lemma~\ref{lem:ahlfors}), which states that $|f'(a)| \le M/r$ when $|f(z)|\le M$ for every $z$ on the circle centered at $a$ with radius $r$. (Note that this strongly relies on the analyticity of $f$ on the complex plane, as boundedness of the function value does not imply boundedness of the derivative for smooth  functions over real numbers.) In our case, we have $a=0$ and $M=r=1$, which means that
\begin{align*}
    \Big|\sum \limits_{t=1}^{T+1} y_{T+1}  \cdots y_{t+1} (1-y_t)  y_{t-1} \cdots y_1\Big| = |f'(0)| \le 1
\end{align*}
The full proof in Appendix \ref{a:pathwise} uses the same idea but extends it to the matrix version.
\end{proof}

To bound the quadratic term  $\|\sum_{t=1}^{T+1} A_tXA_t^\top\|_\op$, we use the following second-order bound for the matrices $A_t$.
\begin{lemma}[Second-order bound]\label{lem:second-order}
    Let $B_1,\ldots,B_{T+1}\in\mathcal S_n$ be positive contractions, and let $A_t= O_t-W$, where $ O_t$ and $W$ are defined as in Lemma~\ref{lem:pathwise}. Then,
    \begin{align*}
        \Big\|\sum_{t=1}^{T+1} A_tA_t^\top\Big\|_\op \leq 1.
    \end{align*}
\end{lemma}
\begin{proof}
To explain the intuition of the proof, we again consider the scalar case first. In this case, for each $t$, we consider \begin{align*}a_t=y_{T+1} \cdots  y_{t+1} (1-y_t)  y_{t-1} \cdots y_1 \le y_{T+1} \cdots (1-y_t),\end{align*} and therefore using $(1-y_t)^2 \le 1-y_t^2$ when $0 \le y_t  \le 1$ we have
\begin{align*}
    a_t^2=y_{T+1}^2 \cdots y_{t+1}^2 (1-y_t)^2 \le y_{T+1}^2 \cdots y_{t+1}^2 (1-y_t^2)= y_{T+1}^2 \cdots y_{t+1}^2 - y_{T+1}^2 \cdots y_{t}^2.
\end{align*}
As a result, $\sum_{t=1}^{T+1} a_t^2$ can be bounded by a telescoping sum as follows:
\begin{align*}
    \sum_{t=1}^{T+1} a_t^2 &\le  1-y_{T+1}^2 +y_{T+1}^2-y_{T+1}^2y_{T}^2+y_{T+1}^2y_{T}^2-y_{T+1}^2y_{T}^2y_{m-2}^2 + \cdots
    \\ &=1-y_{T+1}^2y_{T}^2y_{T-1}^2 \cdots y_{1}^2 \le 1,
\end{align*}
because the middle terms are all canceled out.

Now we explain the formal proof which uses the matrix version of the above idea. For $1\le t\le T+1$, define
\begin{align*}
    L_t=B_{T+1}\cdots B_{t+1}
\quad\text{and}\quad    R_t=B_{t-1}\cdots B_1
\end{align*}
with the convention that an empty product is $I$ (e.g., $R_0=B_{-1}\cdots B_1=I$). So $L_0=W$ and
$L_{T+1}=I$.
And with this notation we have $ O_t=L_tR_t$, $W=L_tB_tR_t$, and  
$A_t=L_t(I-B_t)R_t$. 

 Since each $R_t$ is a product of contractions, we have $\|R_t\|_\op\le1$, or equivalently, $R_tR_t^\top\preceq I_n$.  
 Therefore, using that $M\preceq N$ implies $SMS^\top\preceq SNS^\top$ in the Loewner ordering for any compatible matrix $S$, we obtain
\[
\begin{aligned}
        A_tA_t^\top
        &=L_t(I-B_t)R_tR_t^\top(I-B_t)L_t^\top 
        \preceq L_t(I-B_t)^2L_t^\top.
\end{aligned}
\]
Next, we will use the fact that, just like in the scalar case,
\begin{align*}
    (I-B_t)^2 \preceq I- B_t^2.
\end{align*}
This follows since $B_t$ and $I-B_t$ are commuting positive semidefinite matrices, and therefore $I-B_t^2-(I-B_t)^2=2B_t(I-B_t)\succeq 0$. Thus, we obtain:
\[
\begin{aligned}
        A_tA_t^\top&\preceq L_t(I-B_t^2)L_t^\top
        =L_tL_t^\top-L_tB_t^2L_t^\top 
        =L_tL_t^\top-L_{t-1}L_{t-1}^\top .
\end{aligned}
\]
Summing in $t$ gives the non-commutative telescoping sum
\begin{align*}
        \sum_{t=1}^{T+1} A_tA_t^\top
        &\preceq
        \sum_{t=1}^{T+1}
        \left(L_tL_t^\top-L_{t-1}L_{t-1}^\top\right) \\
        &=L_{T+1}L_{T+1}^\top-L_0L_0^\top \\
        &=I-WW^\top
        \preceq I,
\end{align*}
which concludes the proof of Lemma \ref{lem:second-order}.
\end{proof}
To turn this into a bound on $\|\sum_{t=1}^{T+1}A_tXA_t^\top\|_\op$, we
observe that $-\|X\|_\op I\preceq X\preceq \|X\|_\op I$ for any
$X\in\mathcal S_n$, which implies that:
\[
        -\|X\|_\op A_tA_t^\top
        \preceq A_tXA_t^\top
        \preceq \|X\|_\op A_tA_t^\top.
\]
Putting this together with Lemma \ref{lem:second-order}, we get
\[
        \Big\|\sum_{t=1}^{T+1} A_tXA_t^\top\Big\|_\op 
\le \|X \|_\op \cdot\Big\|\sum_{t=1}^{T+1} A_tA_t^\top\Big\|_\op \leq\|X \|_\op .
\]

Combining \eqref{eq:decomposition} with the first-order and second-order bounds yields Lemma \ref{lem:pathwise}:
\begin{align*}
        \Big\|\sum_{t=1}^{T+1}( O_tX O_t^\top-WXW^\top)\Big\|_\op 
        \le 2\|W\|_\op  \|X\|_\op  \Big\|\sum_{t=1}^{T+1} A_t\Big\|_\op 
             +\Big\|\sum_{t=1}^{T+1} A_tXA_t^\top\Big\|_\op 
  \le 3\|X\|_\op .
\end{align*}

\section{Conclusions}
\label{s:conclusions}

We gave an optimal worst-case complexity analysis for the Kaczmarz
algorithm with randomized equation selection, settling questions posed independently in
applied mathematics \cite{steinerberger2023approximate} and machine learning \cite{attia2025fast}
literature. The Kaczmarz algorithm is the earliest known precursor of 
the SGD methods which are commonly used to train AI models, and our result
extends to a broader class of SGD, where it resolves a related open
question \cite{berthier2020tight}. Remarkably, the proof of the result is due entirely
to a collaboration between AI models ChatGPT and Gemini, which was
merely facilitated by the authors. To put this in a broader
context, we address the following natural questions:

\paragraph{What is the significance of this AI result?}
One can easily argue that this result rises to the level of
novel academic research in its area: It not only resolves questions posed
in works previously published at top-tier venues \cite{steinerberger2023approximate,attia2025fast,berthier2020tight}, but also, it significantly improves on
recently published efforts toward resolving these questions \cite{varre2021last,evron2025continual,attia2025fast,derezinski2026last}, while using an
argument that substantially differs from those prior~works. That
being said, given the recent active progress on this topic and a
relatively elementary nature of the argument, it 
is very likely that human researchers could have settled the
problem with sufficient time and effort.

    \paragraph{What does it say about the capabilities of AI?}
   While it is impossible to fully measure this, one could argue that this problem was near the upper limit of  autonomous research ability
    of the publicly available AI models at the time we first
posed it to them. This is
because neither ChatGPT nor Gemini were able to 
fully resolve it initially, but then they produced a complete solution
after we facilitated a transfer of ideas between them. Therefore,
this case study provides a useful data point for marking the progress in AI
research capabilities, alongside recent related efforts \cite{jang2025point,woodruff2026accelerating,abouzaid2026first,alon2026remarks}.
\bigskip

It is clear that AI is already playing a critical role in the scientific
research workflow. As we move forward, it is important to explore
different forms that this human-AI collaboration takes, so that we can
maximize the benefits while mitigating the risks, and in particular, make sure that
human researchers actually stay in the loop of scientific
progress. We believe that case studies such as this one are an
important part of this process, as they provide
valuable insight into the future of AI in academia.

\subsection*{Acknowledgments}
This work was supported in part by NSF CAREER Grant CCF-233865 and a Google ML and Systems Junior Faculty Award.

\appendix

\section{Matrix-valued leave-one-out bound}
\label{a:pathwise}
Here, we provide the complete proof of Lemma \ref{lem:first-coefficient},
extending the complex analysis argument given  in Section \ref{s:pathwise} from scalars to
matrices. 
\subsection{Preliminaries}
We use the following statement of Cauchy's estimate (e.g., p.~122, equation~(25) in \cite{Ahlfors}). 
\begin{lemma}[Cauchy's estimate]\label{lem:ahlfors}
Let $C\subseteq\bC$ be a circle with center $a$ and radius $r$, and let $f$ be
analytic in a connected open set containing $C$ and its inside.  If
$|f(\zeta)|\le M$ on $C$, then for any positive integer $n$,
\begin{equation}\label{eq:ahlfors25}
        |f^{(n)}(a)|\le M n! r^{-n}.
\end{equation}

\end{lemma}

We also need the following complex matrix differentiation rules, which
we prove here for the sake of completeness.

\begin{lemma}[Complex matrix differentiation rules]\label{lem:complex-matrix-calculus}
Let $\Omega\subset\bC$ be open.  A function $G:\Omega\to M_n(\bC)$ is called complex differentiable at $z\in\Omega$ if the limit
\begin{equation}\label{eq:complex-matrix-derivative-def}
        G'(z)=\lim_{\substack{h\to0,\ h\in\bC\setminus\{0\}\\ z+h\in\Omega}}
        \frac{G(z+h)-G(z)}{h}
\end{equation}
exists in the topology defined by spectral norm.  With this definition, the following differentiation rules hold.
\begin{enumerate}
\item If $G,H:\Omega\to M_n(\bC)$ are complex differentiable at $z$, then $GH$ is complex differentiable at $z$ and
\begin{equation}\label{eq:complex-matrix-product-rule}
        (GH)'(z)=G'(z)H(z)+G(z)H'(z).
\end{equation}
Consequently, if $G_1,\ldots,G_m$ are complex differentiable at $z$, then
\begin{equation}\label{eq:complex-multiple-product-rule}
        (G_m\cdots G_1)'(z)
        =\sum_{j=1}^m
        G_m(z)\cdots G_{j+1}(z)G_j'(z)G_{j-1}(z)\cdots G_1(z),
\end{equation}
where an empty product is interpreted as $I_n$.
\item If $\lambda:M_n(\bC)\to\bC$ is a fixed complex-linear functional and $G$ is complex differentiable at $z$, then $\lambda\circ G$ is complex differentiable at $z$ and
\begin{equation}\label{eq:complex-linear-functional-rule}
        (\lambda\circ G)'(z)=\lambda(G'(z)).
\end{equation}
\end{enumerate}
\end{lemma}

\begin{proof}
First note that complex differentiability at $z$ implies continuity at $z$, since
\[
        G(z+h)-G(z)
        =h\,\frac{G(z+h)-G(z)}{h}\to0
        \qquad(h\to0).
\]
For the product rule, write
\[
\begin{aligned}
        &\frac{G(z+h)H(z+h)-G(z)H(z)}{h} \\
        &\qquad=
        \frac{G(z+h)-G(z)}{h}\,H(z+h)
        +G(z)\frac{H(z+h)-H(z)}{h}.
\end{aligned}
\]

Therefore, the difference between $\frac{G(z+h)H(z+h)-G(z)H(z)}{h}$ and the desired quantity $G'(z)H(z)+G(z)H'(z)$ is
\begin{align*}
        &\frac{G(z+h)-G(z)}{h}\,H(z+h)
        +G(z)\frac{H(z+h)-H(z)}{h}G'(z)H(z)+G(z)H'(z)\\=&\Big(\frac{G(z+h)-G(z)}{h}-G'(z)\Big)H(z+h)
        +G(z)\Big(\frac{H(z+h)-H(z)}{h}-H'(z)\Big)\\&+G'(z)\Big(H(z+h)-H(z)\Big).
\end{align*}
Its norm tends to zero by submultiplicativity, because $\frac{G(z+h)-G(z)}{h}\to G'(z)$, $\frac{H(z+h)-H(z)}{h}\to H'(z)$, and $H(z+h)\to H(z)$ by definition of complex derivative.  This proves \eqref{eq:complex-matrix-product-rule}.  Formula \eqref{eq:complex-multiple-product-rule} follows by induction on $m$.

For the linear-functional rule,
\[
        \frac{\lambda(G(z+h))-\lambda(G(z))}{h}
        =\lambda\left(\frac{G(z+h)-G(z)}{h}\right)
        \to \lambda(G'(z)),
\]
because every linear functional on the finite-dimensional space $M_n(\bC)$ is continuous.
\end{proof}

We are also going to need the following simple observation that any complex
interpolation between positive contractions $B$ and $I-B$ is bounded by 1.
\begin{lemma}[Complex interpolation of a positive contraction]\label{lem:factorcontraction}
Let $B\in\mathcal S_n$ satisfy $0\preceq B \preceq I_n$.  If $z\in\bC$ and $|z|\le1$, then, after complexification,
\[
        \|B+z(I_n-B)\|_\op \le1.
\]
\end{lemma}

\begin{proof}
By the real spectral theorem, we can write $B=Q\diag(\lambda_1,\ldots,\lambda_n)Q^\top$ for some real orthogonal $Q$ and numbers $0\le\lambda_i\le1$.  After complexification, $Q$ is unitary and
\[
        B+z(I-B)
        =Q\diag\bigl(\lambda_i+z(1-\lambda_i)\bigr)_{i=1}^nQ^\top.
\]
Therefore, we have
\[
        \|B+z(I-B)\|_\op 
        =\max_i |\lambda_i+z(1-\lambda_i)|
        \le \max_i\{\lambda_i+|z|(1-\lambda_i)\}
        \le1.
\]
\end{proof}

\subsection{Proof of Lemma \ref{lem:first-coefficient}}
For $1\le t\le T+1$ and $z \in \bC$, let
\[
        E_t(z)=B_t+z(I-B_t),
\]
where the real matrices $B_t$ are viewed as complex matrices acting on $\bC^n$. Define
\[
        F(z)=E_{T+1}(z)E_{T}(z)\cdots E_1(z).
\]

The entries of each $E_t(z)$ are affine scalar functions of $z$, the entries of $F(z)$ are obtained from them by finitely many scalar additions and multiplications, and therefore, for fixed $u,v\in\bC^n$, we have
\begin{align*}
        f_{u,v}(z):=u^*F(z)v
=\sum_{a,b=1}^n\overline{u_a}\,F_{ab}(z)v_b
\end{align*}
is a sum of complex polynomials and therefore  is an entire function (which can be proved more formally by induction on $m$). Note that since we are dealing with complex vectors and matrices, in this section alone we will use $u^*$ to denote the conjugate transpose and $\bar u$ to denote the complex conjugate.

In addition, for $|z|\le1$, by Lemma~\ref{lem:factorcontraction}, we have
\[
        \|E_t(z)\|_\op \le1,
        \qquad 1\le r\le m.
\]
By submultiplicativity of the spectral norm, we obtain
\[
        \|F(z)\|_\op \le1 \qquad 
\]
for any $|z|\le1$. 
Therefore, for any unit vectors $u,v\in\bC^n$ and $|z|=1$, we have
\[
        |f_{u,v}(z)|\le \|F(z)\|_\op \le1.
\]
By Cauchy's estimate (Lemma~\ref{lem:ahlfors}) applied to the unit circle with center $0$, radius $1$, and $M=1$, we obtain
\begin{equation}\label{eq:cauchy-scalar-tested}
        |f'_{u,v}(0)|\le1.
\end{equation}
It remains to relate $f'_{u,v}(0)$ to $u^*\big(\sum_{t=1}^{T+1}A_t\big) v$.  For every $z\in\bC$ and every nonzero complex $h$, we calculate
\[
        \frac{E_t(z+h)-E_t(z)}{h}=I-B_t.
\]
Thus $E_t(z)$ is complex differentiable in the sense of \eqref{eq:complex-matrix-derivative-def}, and
\[
        E_t'(z)=I-B_t.
\]
Applying the complex matrix product rule \eqref{eq:complex-multiple-product-rule} to
\[
        F(z)=E_{T+1}(z)E_{T}(z)\cdots E_1(z),
\]
we have
\[
        F'(z)
        =\sum_{t=1}^{T+1}
        E_{T+1}(z)\cdots E_{t+1}(z)(I-B_t)E_{t-1}(z)\cdots E_1(z).
\]
for every $z\in\bC$.
Plugging in $z=0$, the above identity becomes
\begin{equation}\label{eq:F-complex-derivative-S}
        F'(0)
        =\sum_{t=1}^{T+1}
          B_{T+1}\cdots B_{t+1}(I-B_t)B_{t-1}\cdots B_1
        =\sum_{i=1}^{T+1}A_t.
\end{equation}
Finally, the fixed functional
\[
        \lambda_{u,v}:M_n(\bC)\to\bC,
        \qquad
        \lambda_{u,v}(T)=u^*Tv
\]
is complex-linear and continuous.  By \eqref{eq:complex-linear-functional-rule},
\begin{equation}\label{eq:fprime-usv}
        f'_{u,v}(0)
        =(\lambda_{u,v}\circ F)'(0)
        =\lambda_{u,v}(F'(0))
        =u^*\Big(\sum_{t=1}^{T+1}A_t\Big) v.
\end{equation}
Combining \eqref{eq:cauchy-scalar-tested} and \eqref{eq:fprime-usv},
we have $|u^*\big(\sum_{t=1}^{T+1}A_t\big) v|\le1$ for all unit
vectors $u,v\in\bC^n$, which implies that $\|\sum_{t=1}^{T+1}A_t\|_\op \le1$.

\bibliographystyle{plain}
\bibliography{pap}

\begin{thebibliography}{10}

\bibitem{abouzaid2026first}
Mohammed Abouzaid, Andrew~J Blumberg, Martin Hairer, Joe Kileel, Tamara~G Kolda, Paul~D Nelson, Daniel Spielman, Nikhil Srivastava, Rachel Ward, Shmuel Weinberger, et~al.
\newblock First proof.
\newblock {\em arXiv preprint arXiv:2602.05192}, 2026.

\bibitem{Ahlfors}
Lars~Valerian Ahlfors.
\newblock {\em Complex analysis: an introduction to the theory of analytic functions of one complex variable}.
\newblock McGraw-Hill New York, 1960.

\bibitem{alexeev2026forbidden}
Boris Alexeev and Dustin~G Mixon.
\newblock Forbidden sidon subsets of perfect difference sets, featuring a human-assisted proof.
\newblock {\em Proceedings of the National Academy of Sciences}, 123(21):e2531760123, 2026.

\bibitem{alon2026remarks}
Noga Alon, Thomas~F Bloom, WT~Gowers, Daniel Litt, Will Sawin, Arul Shankar, Jacob Tsimerman, Victor Wang, and Melanie~Matchett Wood.
\newblock Remarks on the disproof of the unit distance conjecture.
\newblock {\em arXiv preprint arXiv:2605.20695}, 2026.

\bibitem{asi2019importance}
Hilal Asi and John~C Duchi.
\newblock The importance of better models in stochastic optimization.
\newblock {\em Proceedings of the National Academy of Sciences}, 116(46):22924--22930, 2019.

\bibitem{attia2025fast}
Amit Attia, Matan Schliserman, Uri Sherman, and Tomer Koren.
\newblock Fast last-iterate convergence of sgd in the smooth interpolation regime.
\newblock {\em Advances in Neural Information Processing Systems}, 38:104951--104987, 2025.

\bibitem{bach2013non}
Francis Bach and Eric Moulines.
\newblock Non-strongly-convex smooth stochastic approximation with convergence rate o (1/n).
\newblock {\em Advances in neural information processing systems}, 26, 2013.

\bibitem{berthier2020tight}
Rapha{\"e}l Berthier, Francis Bach, and Pierre Gaillard.
\newblock Tight nonparametric convergence rates for stochastic gradient descent under the noiseless linear model.
\newblock {\em Advances in Neural Information Processing Systems}, 33:2576--2586, 2020.

\bibitem{bi2024deepseek}
Xiao Bi, Deli Chen, Guanting Chen, Shanhuang Chen, Damai Dai, Chengqi Deng, Honghui Ding, Kai Dong, Qiushi Du, Zhe Fu, et~al.
\newblock Deepseek llm: Scaling open-source language models with longtermism.
\newblock {\em arXiv preprint arXiv:2401.02954}, 2024.

\bibitem{blunck2001analyticity}
S{\"o}nke Blunck.
\newblock Analyticity and discrete maximal regularity on lp-spaces.
\newblock {\em Journal of Functional Analysis}, 183(1):211--230, 2001.

\bibitem{bottou2018optimization}
L{\'e}on Bottou, Frank~E Curtis, and Jorge Nocedal.
\newblock Optimization methods for large-scale machine learning.
\newblock {\em SIAM review}, 60(2):223--311, 2018.

\bibitem{brown2020language}
Tom Brown, Benjamin Mann, Nick Ryder, Melanie Subbiah, Jared~D Kaplan, Prafulla Dhariwal, Arvind Neelakantan, Pranav Shyam, Girish Sastry, Amanda Askell, et~al.
\newblock Language models are few-shot learners.
\newblock {\em Advances in neural information processing systems}, 33:1877--1901, 2020.

\bibitem{derezinski2026last}
Micha{\l} Derezi{\'n}ski and Xiaoyu Dong.
\newblock Last-iterate convergence of randomized kaczmarz and sgd with greedy step size.
\newblock {\em Conference on Learning Theory}, 2026.

\bibitem{derezinski2025randomized}
Micha{\l} Derezi{\'n}ski, Deanna Needell, Elizaveta Rebrova, and Jiaming Yang.
\newblock Randomized kaczmarz methods with beyond-krylov convergence.
\newblock {\em SIAM Journal on Matrix Analysis and Applications}, 46(4):2558--2588, 2025.

\bibitem{duchi2018introductory}
John~C Duchi.
\newblock Introductory lectures on stochastic optimization.
\newblock {\em The mathematics of data}, 25:99--186, 2018.

\bibitem{evron2025continual}
Itay Evron, Ran Levinstein, Matan Schliserman, Uri Sherman, Tomer Koren, Daniel Soudry, and Nathan Srebro.
\newblock From continual learning to sgd and back: Better rates for continual linear models.
\newblock In {\em Fourth Conference on Lifelong Learning Agents-Workshop Track}, 2025.

\bibitem{garrigos2023handbook}
Guillaume Garrigos and Robert~M Gower.
\newblock Handbook of convergence theorems for (stochastic) gradient methods.
\newblock {\em arXiv preprint arXiv:2301.11235}, 2023.

\bibitem{ge2019step}
Rong Ge, Sham~M Kakade, Rahul Kidambi, and Praneeth Netrapalli.
\newblock The step decay schedule: A near optimal, geometrically decaying learning rate procedure for least squares.
\newblock {\em Advances in neural information processing systems}, 32, 2019.

\bibitem{gower2015randomized}
Robert~M Gower and Peter Richt{\'a}rik.
\newblock Randomized iterative methods for linear systems.
\newblock {\em SIAM Journal on Matrix Analysis and Applications}, 36(4):1660--1690, 2015.

\bibitem{harvey2019tight}
Nicholas~JA Harvey, Christopher Liaw, Yaniv Plan, and Sikander Randhawa.
\newblock Tight analyses for non-smooth stochastic gradient descent.
\newblock In {\em Conference on Learning Theory}, pages 1579--1613. PMLR, 2019.

\bibitem{jain2019making}
Prateek Jain, Dheeraj Nagaraj, and Praneeth Netrapalli.
\newblock Making the last iterate of sgd information theoretically optimal.
\newblock In {\em Conference on Learning Theory}, pages 1752--1755. PMLR, 2019.

\bibitem{jang2025point}
Uijeong Jang and Ernest~K Ryu.
\newblock Point convergence of nesterov's accelerated gradient method: An ai-assisted proof.
\newblock {\em arXiv preprint arXiv:2510.23513}, 2025.

\bibitem{junge2006h}
Marius Junge, Christian Le~Merdy, and Quanhua Xu.
\newblock {\em H$^\infty$ functional calculus and square functions on noncommutative Lp-spaces}.
\newblock Soci{\'e}t{\'e} math{\'e}matique de France, 2006.

\bibitem{kaczmarz37}
M.~S. Kaczmarz.
\newblock Angenaherte auflosung von systemen linearer gleichungen.
\newblock {\em Bulletin International de l’Academie Polonaise des Sciences et des Lettres}, 35:355–357, 1937.

\bibitem{kingma2014adam}
Diederik~P Kingma and Jimmy Ba.
\newblock Adam: A method for stochastic optimization.
\newblock {\em International Conference on Learning Representations}, 2015.

\bibitem{lee2013efficient}
Yin~Tat Lee and Aaron Sidford.
\newblock Efficient accelerated coordinate descent methods and faster algorithms for solving linear systems.
\newblock In {\em 2013 ieee 54th annual symposium on foundations of computer science}, pages 147--156. IEEE, 2013.

\bibitem{leventhal2010randomized}
Dennis Leventhal and Adrian~S Lewis.
\newblock Randomized methods for linear constraints: convergence rates and conditioning.
\newblock {\em Mathematics of Operations Research}, 35(3):641--654, 2010.

\bibitem{liu2023aiming}
Chaoyue Liu, Dmitriy Drusvyatskiy, Misha Belkin, Damek Davis, and Yian Ma.
\newblock Aiming towards the minimizers: fast convergence of sgd for overparametrized problems.
\newblock {\em Advances in neural information processing systems}, 36:60748--60767, 2023.

\bibitem{liu2016accelerated}
Ji~Liu and Stephen Wright.
\newblock An accelerated randomized kaczmarz algorithm.
\newblock {\em Mathematics of Computation}, 85(297):153--178, 2016.

\bibitem{liu2023revisiting}
Zijian Liu and Zhengyuan Zhou.
\newblock Revisiting the last-iterate convergence of stochastic gradient methods.
\newblock {\em The Twelfth International Conference on Learning Representations}, 2023.

\bibitem{needell2014stochastic}
Deanna Needell, Nathan Srebro, and Rachel Ward.
\newblock Stochastic gradient descent, weighted sampling, and the randomized kaczmarz algorithm.
\newblock {\em Advances in neural information processing systems}, 27, 2014.

\bibitem{needell2014paved}
Deanna Needell and Joel~A Tropp.
\newblock Paved with good intentions: analysis of a randomized block kaczmarz method.
\newblock {\em Linear Algebra and its Applications}, 441:199--221, 2014.

\bibitem{unit-distances}
OpenAI.
\newblock Planar point sets with many unit distances.
\newblock 2026.
\newblock Manuscript available at https://cdn.openai.com/pdf/74c24085-19b0-4534-9c90-465b8e29ad73/unit-distance-proof.pdf.

\bibitem{robbins1951stochastic}
Herbert Robbins and Sutton Monro.
\newblock A stochastic approximation method.
\newblock {\em The annals of mathematical statistics}, pages 400--407, 1951.

\bibitem{saad2003iterative}
Yousef Saad.
\newblock {\em Iterative methods for sparse linear systems}.
\newblock SIAM, 2003.

\bibitem{shamir2012open}
Ohad Shamir.
\newblock Open problem: Is averaging needed for strongly convex stochastic gradient descent?
\newblock In {\em Conference on Learning Theory}, pages 47--1. JMLR Workshop and Conference Proceedings, 2012.

\bibitem{shamir2013stochastic}
Ohad Shamir and Tong Zhang.
\newblock Stochastic gradient descent for non-smooth optimization: Convergence results and optimal averaging schemes.
\newblock In {\em International conference on machine learning}, pages 71--79. PMLR, 2013.

\bibitem{shen1999nine}
Kangshen Shen, John~N Crossley, and Anthony Wah-Cheung Lun.
\newblock {\em The nine chapters on the mathematical art: Companion and commentary}.
\newblock OUP Oxford, 1999.

\bibitem{steinerberger2023approximate}
Stefan Steinerberger.
\newblock Approximate solutions of linear systems at a universal rate.
\newblock {\em SIAM Journal on Matrix Analysis and Applications}, 44(3):1436--1446, 2023.

\bibitem{strohmer2009randomized}
Thomas Strohmer and Roman Vershynin.
\newblock A randomized kaczmarz algorithm with exponential convergence.
\newblock {\em Journal of Fourier Analysis and Applications}, 15(2):262--278, 2009.

\bibitem{varre2021last}
Aditya~Vardhan Varre, Loucas Pillaud-Vivien, and Nicolas Flammarion.
\newblock Last iterate convergence of sgd for least-squares in the interpolation regime.
\newblock {\em Advances in Neural Information Processing Systems}, 34:21581--21591, 2021.

\bibitem{vaswani2019fast}
Sharan Vaswani, Francis Bach, and Mark Schmidt.
\newblock Fast and faster convergence of sgd for over-parameterized models and an accelerated perceptron.
\newblock In {\em The 22nd international conference on artificial intelligence and statistics}, pages 1195--1204. PMLR, 2019.

\bibitem{woodruff2026accelerating}
David~P Woodruff, Vincent Cohen-Addad, Lalit Jain, Jieming Mao, Song Zuo, MohammadHossein Bateni, Simina Branzei, Michael~P Brenner, Lin Chen, Ying Feng, et~al.
\newblock Accelerating scientific research with gemini: Case studies and common techniques.
\newblock {\em arXiv preprint arXiv:2602.03837}, 2026.

\bibitem{wu2022last}
Jingfeng Wu, Difan Zou, Vladimir Braverman, Quanquan Gu, and Sham Kakade.
\newblock Last iterate risk bounds of sgd with decaying stepsize for overparameterized linear regression.
\newblock In {\em International conference on machine learning}, pages 24280--24314. PMLR, 2022.

\bibitem{zamani2025exact}
Moslem Zamani and Francois Glineur.
\newblock Exact convergence rate of the last iterate in subgradient methods.
\newblock {\em SIAM Journal on Optimization}, 35(3):2182--2201, 2025.

\bibitem{zouzias2013randomized}
Anastasios Zouzias and Nikolaos~M Freris.
\newblock Randomized extended kaczmarz for solving least squares.
\newblock {\em SIAM Journal on Matrix Analysis and Applications}, 34(2):773--793, 2013.

\end{thebibliography}

\end{document}